\documentclass[letterpaper]{article} % DO NOT CHANGE THIS
\usepackage[]{aaai25}  % DO NOT CHANGE THIS
\usepackage{times}  % DO NOT CHANGE THIS
\usepackage{helvet}  % DO NOT CHANGE THIS
\usepackage{courier}  % DO NOT CHANGE THIS
\usepackage[hyphens]{url}  % DO NOT CHANGE THIS
\usepackage{graphicx} % DO NOT CHANGE THIS
\usepackage{natbib}  % DO NOT CHANGE THIS AND DO NOT ADD ANY OPTIONS TO IT
\usepackage{caption} % DO NOT CHANGE THIS AND DO NOT ADD ANY OPTIONS TO IT
\usepackage{algorithm,algpseudocode}

\makeatletter
\newcommand{\StatexIndent}[1][3]{%
  \setlength\@tempdima{\algorithmicindent}%
  \Statex\hskip\dimexpr#1\@tempdima\relax}
\makeatother

\usepackage{newfloat}
\usepackage{listings}
\DeclareCaptionStyle{ruled}{labelfont=normalfont,labelsep=colon,strut=off} % DO NOT CHANGE THIS
\floatstyle{ruled}
\newfloat{listing}{tb}{lst}{}
\floatname{listing}{Listing}
\nocopyright %-- Your paper will not be published if you use this command
\usepackage{amsmath}
\usepackage{graphicx}

\usepackage{amssymb}
\usepackage{amsthm}
\usepackage{trimclip}
\usepackage{xcolor}
\usepackage{cleveref}
\usepackage{tikz}
\usetikzlibrary{positioning}
\usepackage[export]{adjustbox}
\usepackage{stackengine}
\usepackage{mathtools}
\usepackage{caption}
\usepackage{subcaption}
\usepackage{MnSymbol}
\usepackage{ wasysym }
\usepackage{scalerel}
\usepackage{cancel}

\let\bigopsize\bigcirc

\def\bigominus{{\scalerel*{\ominus}{\bigopsize}}}

\newtheorem{theorem}{Theorem}
\newtheorem{lemma}{Lemma}
\newtheorem{corollary}{Corollary}
\newtheorem{definition}{Definition}

\newtheorem{observation}{Observation}

\newtheorem{example}{Example}

\newcommand{\cD}{\mathcal{D}}
\newcommand{\cM}{\mathcal{M}}

\newcommand{\cF}{\mathcal{F}}
\newcommand{\cS}{\mathcal{S}}

\newcommand{\cV}{\mathcal{V}}
\newcommand{\cR}{\mathcal{R}}

\newcommand{\cC}{\mathcal{C}}

\newcommand{\cU}{\mathcal{U}}

\newcommand{\cO}{\mathcal{O}}

\newcommand{\U}{\mathsf{U}}

\newcommand{\G}{\mathsf{G}}
\newcommand{\F}{\mathsf{F}}

\newcommand{\sS}{\mathsf{S}}

\newcommand{\LTL}{\mathsf{LTL}}
\newcommand{\PLTL}{\mathsf{PLTL}}
\newcommand{\CPLTL}{\mathsf{CPLTL}}

\usetikzlibrary{chains,fit,shapes}

\tikzset{%
round/.style={circle, draw=gray!60,fill=gray!5, very thick,minimum size=7mm},
dot/.style={draw, circle, minimum size=2mm,inner sep=0pt,outer sep=0pt,fill=black},% and so on
}



\usepackage{soul}
\usepackage{todonotes}

\algtext*{EndIf}
\algtext*{EndFor}

\algnewcommand\algorithmiccase{\textbf{case}}
\algdef{SE}[CASE]{Case}{EndCase}[1]{\algorithmiccase\ #1}{\algorithmicend\ \algorithmiccase}%
\algtext*{EndCase}

\title{Temporal Causal Reasoning with (Non-Recursive) Structural Equation Models\footnote{
This is an extended version of the same title paper published in the proceeding of AAAI-25 conference \cite{Gladyshev_AAAI2025}. This version contains an additional section, in which we introduce TSEMs with arbitrary long temporal delays between causal dependencies and prove that these models are equivalent to TSEMs with 1-step delays only introduced in the original version of the paper.}
}
\author{
    Maksim Gladyshev\textsuperscript{\rm 1},
    Natasha Alechina\textsuperscript{\rm 2,1},
    Mehdi Dastani\textsuperscript{\rm 1},
    Dragan Doder\textsuperscript{\rm 1},
    Brian Logan\textsuperscript{\rm 3,1}
}
\affiliations{
    \textsuperscript{\rm 1} Utrecht University, The Netherlands\\
    \textsuperscript{\rm 2} Open Universiteit, The Netherlands\\
    \textsuperscript{\rm 3} University of Aberdeen, UK
    \{m.gladyshev, n.a.alechina, m.m.dastani, d.doder,  b.s.logan\}@uu.nl 
}

\begin{document}

\maketitle

\begin{abstract}
% Structural equation models are used to represent causal dependencies between variables. The standard assumption is that they are recursive, that is, there are no cycles in the dependency graphs. We consider non-recursive models that can be naturally interpreted as giving rise to infinite sequences of assignments of values to variables. We introduce a temporal logic for causal reasoning about such structures and prove that its model-checking is in polynomial time. We also introduce the notions of model equivalence for non-recursive models.  

Structural Equation Models (SEM) are the standard approach to representing causal dependencies between variables in causal models. In this paper we propose a new interpretation of SEMs when reasoning about Actual Causality, in which SEMs are viewed as mechanisms transforming the dynamics of exogenous variables into the dynamics of endogenous variables. This allows us to combine counterfactual causal reasoning with existing temporal logic formalisms, and to introduce a temporal logic, $\CPLTL$, for causal reasoning about such structures. We show that the standard restriction to so-called \textit{recursive} models (with no cycles in the dependency graph) is not necessary in our approach, allowing us to reason about mutually dependent processes and feedback loops. Finally, introduce new notions of model equivalence for temporal causal models, and show that $\CPLTL$ has an efficient model-checking procedure. 

\end{abstract}


\section{Introduction}

% \maksim{To do:\\
% Add an application example in the Intro\\
% if have space, mention symbolic interventions (review 3) in FW.
% }

There has recently been increased interest in causal reasoning from AI researchers and philosophers. 
%bsl: I don't know what this means 
%and causal reasoning finds new applications in many domains. 
The standard framework for reasoning about causal dependencies in both stochastic and deterministic settings is Structural Equation Modelling (SEM) \cite{Pearl2000,Spirtes2001}. Both types of reasoning are crucial for the field of AI: in stochastic domains it is often used in Causal Machine Learning \cite{causal_timeseries} and Causal Discovery, while in deterministic domains it forms the basis for work on Actual Causality \cite{HalpernBook}. %\alert{Fortunately, both disciplines achieved significant results in recent years and made SEM a well-studied formalizm.} 

However, reasoning about many real-world phenomena and their dynamics requires reasoning about time and temporal properties. As a result, temporal reasoning in, for example, Linear-time Temporal Logic ($\LTL$) \cite{Pnueli1977}, has become an important technique for reasoning about the dynamics of AI systems. Progress in this field has resulted in many theoretical results in areas such as formal verification and synthesis \cite{GorankoBook}. 

While structural equation models are the main tool for analysing cause–effect relations and have been applied in  a range of disciplines, e.g., medicine, economics, computer science and industrial engineering, they are not specifically designed for representing the temporal behaviour of a system, which play important role for causality claims in many domains. For example, the effect of some treatment may (causally) depend not only on whether the treatment was given to a patient or not, but also on temporal properties (i.e, dynamics) of this treatment, such as timing, duration and repetition of the treatment. 
Combining SEM models with temporal reasoning is therefore key in many applications. 

Previous work on combining SEM models with temporal reasoning has focused on causal discovery, and several methods for analysing time-series data with SEMs have been developed, e.g., \cite{Assaad_2022,hyvärinen2023}. In this paper, we propose an approach to temporal reasoning with SEMs for reasoning about actual causality. While a  causal model is usually understood as a static representation of causal dependencies transforming values of exogenous variables into the values of endogenous variables, we show it can also be interpreted as a causal mechanism transforming the \textit{dynamics} of exogenous changes into the \textit{dynamics} of endogenous ones. In our framework, we assume that input to a causal model is a (time) series of values assigned to the exogenous variables, which we call the `temporal context'. We give a procedure that, given a causal model as input, processes a temporal context and transforms it into a (time) series of assignments to endogenous variables. We then show how the framework of actual causality can be combined with the temporal logic $\LTL$ to give the logic $\CPLTL$, allowing us to express statements about future and past of the system, e.g., ``a fact $\varphi$ was \textit{always} true", ``a fact $\varphi$ will be true \textit{until} another fact $\psi$ is true", etc. 

Our framework has several interesting features. Firstly, interventions (necessary for counterfactual reasoning) become `time-sensitive': in temporal settings it is necessary to specify not only which intervention happens, but also \textit{when} it happens. Secondly, most existing works on actual causality only deal with \textit{recursive} causal models (models with acyclic dependency graphs). In our approach cycles in the dependency graph have a natural temporal interpretation, so they do not create technical difficulties, but instead provide useful modelling tools. Following \citet{Beckers2021eq}, we introduce new notions of (temporal) equivalence for causal models, which also covers non-recursive cases. Finally, we show that our framework has an efficient model-checking procedure. 

 %bsl: this seems to belong in the conclusion?
 %While we only consider deterministic structural equations in this paper, we believe our approach may be applicable to stochastic causal modelling. 

\section{Formal Background}

In this section we introduce the formal apparatus we use in the rest of the paper: Structural Equation Models (SEM's), also called Causal Models, used for modelling causal dependencies between events, and Linear-time Temporal Logic ($\LTL$) designed for temporal reasoning.

\subsection{Structural Equation Models}\label{sec:SEM's}

The presentation below essentially follows \cite{HalpernBook}.
Let $\mathcal{U}$ and $\mathcal{V}$ be the finite sets of \textit{exogenous} and \textit{endogenous} variables respectively. We say that $\cS = (\mathcal{U}, \mathcal{V}, \mathcal{R})$ is a \textit{signature}, where $\mathcal{R}: \mathcal{U} \cup \mathcal{V} \to 2^{\mathbb{R}}$ associates with every variable $Y\in \mathcal{U} \cup \mathcal{V}$ a
non-empty \textit{finite} set $\mathcal{R}(Y)$ of possible values, also called \textit{range} of $Y$. 

% \maksim{@Mehdi, I use $\mathbb{R}$ instead of $\mathbb{N}$ because in some examples I want to use $.5$ or similar values. This of course can be expressed with $\mathbb{N}$, but takes more space to explain. Secondly, I reviewed Halpern's paper where he defines it also with $\mathbb{R}$. So, until we require $\cR$ to be finite, I see no difference between $\mathbb{N}$ and $\mathbb{R}$.}

\begin{definition}[Causal model]\label{def:causalmodel} Causal Model (or SEM) over $\mathcal{S}$ is a tuple $\cM = (\mathcal{S}, \cF)$, where $\cF$ associates with every endogenous variable $X \in \mathcal{V}$ a function $$\mathcal{F}_X:\prod\limits_{Z \in (\cU \cup \cV)}\cR(Z)\to\cR(X)$$which defines the structural equation describing how the value of $X$ depends on the values of $\cU\cup\cV$. 
\end{definition}
% \suggestion{
% \textbf{Dragan:}   $\mathcal{F}_X$ is a function, not an equation (and here it is never related to the equations (eg, from Example 1)). How about modifying the second part of Def 1 to something like:\\

% \dots $\cF$ associates with every endogenous variable $X \in \mathcal{V}$ a function $\mathcal{F}_X:\prod\limits_{Z \in (\cU \cup \cV)}\cR(Z)\to\cR(X)$, which defines the structural equation $X:=\mathcal{F}_X (\cU, \cV)$ describing how the value of $X$ depends on the values of $\cU\cup\cV$.

% (of course, $\mathcal{F}_X (\cU, \cV)$ is a bit imprecise, that is why Halpern  uses examples to introduce structural equations)}

% \maksim{Yes, it sounds a bit  ugly, but some authors use the same expressions. For example \cite{Beckers2021eq}, but I agree with the point. Maybe avoid this notation and write just 

% \dots $\cF$ associates with every endogenous variable $X \in \mathcal{V}$ a function $\mathcal{F}_X:\prod\limits_{Z \in (\cU \cup \cV)}\cR(Z)\to\cR(X)$, which defines the structural equation describing how the value of $X$ depends on the values of $\cU\cup\cV$. ?
% }

Informally, in causal models different events are represented by the assignment of different values to abstract variables. Values of \textit{endogenous} variables depend on the values of other variables, while values of \textit{exogenous} variables are determined outside of the model. A complete assignment $(U_1=u_1, \dots, U_k=u_k)$ of $\cU$ is called a \textit{context} and denoted $\vec{u}$. A pair $(\cM, \vec{u})$ is called \textit{causal setting}. 
% \textbf{Natasha: why more notation?}
% We also often write $X \fallingdotseq \{x_1, \dots, x_k\}$ instead of $\cR(X)=\{x_1, \dots, x_k\}$. Now, consider a simple example introduced in \cite{where?}.

\begin{example}[Rocks]\label{exm:rocks}  Suzy and Billy both pick up rocks and throw them at a bottle (encoded as ST=1 and BT=1 respectively). Both throws are perfectly accurate, so the bottle shatters (BS=1) whenever ST=1 or BT=1.
\end{example}

Because all the variables are binary here, for simplicity we write $ST$ and $\neg ST$ instead of ST=1 and ST=0 respectively. It is assumed that exogenous variables $U_{ST}$ and $U_{BT}$ determine values of $ST$ and $BT$: $ST:=U_{ST}$, $BT:=U_{BT}$. The structural equation for $BS$ is $BS:= ST\vee BT$. It is often convenient to represent the structure of the model as a dependency graph. Nodes in the graph represent variables, and directed edges represent (direct) dependencies among variables.

\begin{figure}[h!]
\centering
    \begin{tikzpicture}
    \node[dot] [label=right:{BS}] (BS) {};
    \node[dot] [label={0:ST}]  (ST) [above left=.25cm and .8cm of BS] {};
    \node[dot] [label={0:BT}] (BT) [below left=.25cm and .8cm of BS] {};
    \node[dot] [label=left:{UST}] (UST) [left=of ST] {};
    \node[dot] [label=left:{UBT}] (UBT)  [left=of BT]    {};
%Lines
    \draw[-latex] (UST) -- (ST);
    \draw[-latex] (UBT) -- (BT);
    \draw[-latex] (ST) -- (BS);
    \draw[-latex] (BT) -- (BS);
    \end{tikzpicture}
\caption{A dependency graph for \Cref{exm:rocks}.}
\label{fig:pre:rockexample}
\end{figure}
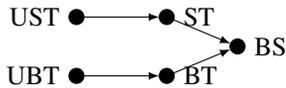

A model $\cM$ is \textit{recursive} if there exists a partial order $\preccurlyeq$ on
$\cV$, such that unless $X_2\preccurlyeq X_1$, $X_1$ is independent of $X_2$ \cite{HalpernBook}. As a result, a dependency graph of a recursive model is a directed acyclic graph. Recursiveness also guarantees that given $\vec{u}$, the set of structural equations has a unique solution, i.e., a unique assignment of $\cV$. For this reason,  many papers on actual causality consider recursive causal models only (e.g., \cite{Halpern2005a,Beckers2018}). As we show later, this restriction recursive models is not necessary in our approach.
%our approach allows to overcome this limitation. 

The last ingredient we need for reasoning about actual and counterfactual courses of events are  \textit{interventions}. An intervention $[\vec{X}\leftarrow \vec{x}]\varphi$, meaning ``after fixing the values of $\vec{X}\subseteq \cV$ to $\vec{x}$, $\varphi$ holds", results in a new casual model denoted $\cM_{\vec{X}\leftarrow\vec{x}}$. $\cM_{\vec{X}\leftarrow\vec{x}}$ is the model $\cM$ where functions $\cF_X$ for any $X\in\vec{X}$ are replaced with a constant function $\cF'_X$, which always returns $x^*$, where $X=x^*\in \vec{X}\leftarrow\vec{x}$ and the remaining functions remain unchanged.

Note that
$\vec{X}$ abbreviates $\{X_1, \dots, X_k\}$; $\vec{X}=\vec{x}$ abbreviates $\{X_1 =x_1, \dots, X_k=x_k\}$; $\vec{X}\leftarrow \vec{x}$ abbreviates $\{X_1 \leftarrow x_1, \dots, X_k\leftarrow x_k\}$. Sometimes we slightly abuse this notation and write $X = x \in \vec{X}\leftarrow\vec{x}$ instead of $X \leftarrow x \in \vec{X}\leftarrow\vec{x}$.

In our example in a context $\vec{u}=(U_{ST}=1, U_{BT}=1)$, where both Suzy and Billy throw their rocks, formulas $ST=1$, $BT=1$, $BS=1$ are true. At the same time, interventions allow us to formulate statements, such that $[ST\leftarrow 0]BS=1$, meaning ``if Suzy does not throw the rock, the bottle still shatters" or $[ST\leftarrow 0, BT\leftarrow 0]\neg (BS=1)$, meaning ``if neither Suzy nor Billy throw the rock, then the bottle is not shattered". Thus, interventions provide us all the necessary machinery for counterfactual reasoning about SEM's.

 %We also allow expressions $\vec{X} = \vec{x}$ to appear in our syntax. In this case, we interpret $\vec{X} = \vec{x}$ as a conjunction $X_1=x_1\wedge \dots \wedge X_k=x_k$. Finally, we use $\vec{u}$ and $\vec{v}$ to denote a complete assignments of $\cU$ and $\cV$ respectively. 

\begin{definition}[Syntax] The grammar of the basic causal language is defined as follows:
 \[\varphi::=  [\vec{Y}\leftarrow \vec{y}]\psi \mid \neg\varphi\mid \varphi\wedge\varphi\]
    \[\psi::= (X=x) \mid \neg \psi \mid \psi\wedge \psi,  \]
where $X\in \cU\cup\cV$, $x\in\cR(X), \vec{Y}\subseteq \cV$ and $\vec{y}\in \cR(\vec{Y})$. Note that $\vec{Y}\leftarrow \vec{y}$ may be empty, so we write $\varphi$ instead of $[]\varphi$. 
\end{definition}
% \textbf{Natasha: why do we need this empty intervention?}
% \maksim{@Natasha,

% Firstly, it's is a bit easier to consider different cases for MC, but also Halper defines the language in the same way \cite{HalpernBook} (pp. 20-21):

% ''A causal formula (over S) is one of the form
% $[Y_1\leftarrow y_1,\dots , Y_k \leftarrow y_k]\varphi$, 

% [...]

% Such a formula is abbreviated as $[\vec{Y} \leftarrow \vec{y}]\varphi$, using the vector notation. The special case where
% $k = 0$ is abbreviated as $[]\varphi$ or, more often, just $\varphi$ "
% }

The \textit{truth} relation $(\cM, \vec{u}) \Vdash \varphi$, meaning that a causal formula $\varphi$ is true in
a causal setting $(\cM, \vec{u})$, is defined inductively as follows:\\
$(\cM, \vec{u})\Vdash (X=x)$ iff $(X=x)\in Sol(\vec{u})$\footnote{Here, $Sol(\vec{u})$ denotes the unique solution of the equations
in $\cM$ in context $\vec{u}$ (existing by the recursiveness of $\cM$). 
};\\
$(\cM, \vec{u})\Vdash \neg\varphi$ iff $(\cM, \vec{u})\nVdash \varphi$; \\
$(\cM, \vec{u})\Vdash (\varphi\wedge\psi)$ iff $(\cM, \vec{u})\Vdash\varphi$ and $(\cM, \vec{u})\Vdash\psi$; \\
$(\cM, \vec{u})\Vdash [\vec{Y}\leftarrow \vec{y}]\varphi$ iff  $(\cM_{\vec{Y}\leftarrow \vec{y}}, \vec{u})\Vdash \varphi$.

\subsection{Linear-time Temporal Logic}

Now we introduce some basics of Linear-time Temporal Logic ($\LTL$), for an extensive overview see \cite{GorankoBook}. In this paper we use both future and past $\LTL$ operators, so we call it $\PLTL$, and it contains four basic modalities: $\bigcirc\varphi$ meaning ``$\varphi$ will be true in the next moment", $\varphi\U\psi$ meaning ``$\varphi$ will be true until $\psi$", $\bigominus\varphi$ meaning ``$\varphi$ was true in the previous moment" and $\varphi\sS\psi$ meaning ``$\varphi$ is true since $\psi$". The only difference of our approach from the standard $\PLTL$ definitions is  that we use atomic expressions $(X=x)$ generated by a given signature $\cS$ instead of atomic propositions $Prop=\{p, q, \dots\}$.

\begin{definition}[$\PLTL$ syntax]\label{def:syntaxPLTL} Given a signature $\cS$, $\PLTL$ syntax is defined as:
\[\varphi ::= (X=x) \mid \neg \varphi \mid \varphi\wedge \varphi \mid \bigcirc\varphi \mid \varphi\U\varphi \mid \bigominus\varphi\mid \varphi\sS\varphi,\]
where $X\in \cV, x\in \cR(X)$. 
\end{definition}

We use standard abbreviations for
other Boolean connectives, together with derived operators $\F
\varphi \equiv \top\U\varphi$ for \textit{eventually}; $\G \varphi\equiv \neg \F\neg\varphi$ for \textit{always} in future; $\mathsf{P}\varphi\equiv \top\sS\varphi$ for \textit{sometime} in the past;  $\mathsf{H}\varphi\equiv \neg \mathsf{P}\neg\varphi$ for \textit{always} in the past.  We refer to the fragments of $\PLTL$ without $\{\bigominus\varphi, \varphi\sS\psi\}$ operators  and without $\{\bigcirc\varphi, \varphi\U\psi \}$ operators as to $\LTL$ and \textit{pure-past} $\LTL$ respectively. We also write $\bigcirc^n$ to abbreviate $\bigcirc$ nested $n$ times. 
The models of $\PLTL$ are infinite sequences of complete assignments to the variables in $\cV$. 

\begin{definition}[Linear model]\label{def:linear_model} For a given signature $\cS$, a linear model is an infinite sequence of assignments of all endogenous variables, i.e.,
\[\sigma: \mathbb{N} \to \prod\limits_{X\in \cV} \cR(X)\]
\end{definition} 

%To illustrate how these models look like, consider a simple example.

\begin{example}[Treatment]\label{exm:treatment} Suppose a patient is ill, and a medication exists. But this medication works only if it is given twice, on two consecutive days. The patient is recovered ($R=1$) at step $i$ if and only if the  treatment was given $(T=1)$ at two consecutive previous steps $i-1$ and $i-2$. Once the patient is recovered, they remain so.
\end{example}

Let us fix $\cS = (\cU, \cV, \cR)$, with $\cV = \{T, R\}$ ($T$ stands for Treatment and $R$ stands for Recovery) and all variables are binary.
Consider two linear models over $\cS$:\\
$\sigma_1 = ((T\neg R), (\neg T\neg R), (T\neg R), (\neg T\neg R), (T\neg R), \dots)$\\
and \\
$\sigma_2 = ((\neg T\neg R), (T\neg R), (T\neg R), (\neg TR), (\neg TR), \dots)$

The first model $\sigma_1$ depicts a situation, when the treatment is given at every even time moment $i$ (including $i=0$) and the patient is never recovered. In the second model the treatment is given twice in a row, at steps $i=1$ and $i=2$, and the patient recovers at $i=3$. Given a model and a time moment, $\PLTL$ logic allows us to express various facts about the past, the present and the future with respect to this time moment. 

\begin{definition}[$\PLTL$ Semantics]\label{def:semanticsPLTL} Given a linear model $\sigma$, a position $i\in \mathbb{N}$ and a formula $\varphi\in \PLTL$, we define the
\emph{truth} relation $\vDash_{\PLTL}$ inductively as follows\\
$(\sigma, i)\vDash_{\PLTL} (X=x)$ iff $(X=x)\in \sigma(i)$;\\
$(\sigma, i)\vDash_{\PLTL} \neg\varphi$ iff $(\sigma, i)\nvDash_{\PLTL} \varphi$; \\
$(\sigma, i)\vDash_{\PLTL} (\varphi\wedge\psi)$ iff $(\sigma, i)\vDash_{\PLTL}\varphi$ and $(\sigma, i)\vDash_{\PLTL}\psi$; \\
$(\sigma, i)\vDash_{\PLTL} \bigcirc\varphi$ iff $(\sigma, i+1)\vDash_{\PLTL} \varphi$;\\
$(\sigma, i)\vDash \varphi\U\psi$ iff there exists $j\geq 0$ such that $(\sigma, i+j)\vDash_{\PLTL} \psi$ and for all $0\leq k < j: (\sigma, i+k)\vDash_{\PLTL} \varphi$;\\
$(\sigma, i)\vDash_{\PLTL} \bigominus\varphi$ iff  $i\geq 1$ and $(\sigma, i-1)\vDash_{\PLTL} \varphi$;\\ 
$(\sigma, i)\vDash_{\PLTL} \varphi\sS\psi$ iff  $\exists k$ with $0\leq k \leq i$ such
that $(\sigma, k)\vDash_{\PLTL} \psi$ and $\forall j$ with $k < j \leq i: (\sigma, j)\vDash_{\PLTL} \varphi$.
\end{definition}

 The following expressions are true about \Cref{exm:treatment}:\\
\textbullet\ $(\sigma_1, 0)\vDash_{\PLTL} \G (\bigominus T \to (\neg T \wedge \bigcirc T)\wedge \neg R)$, meaning that at $(\sigma_1, 0)$  ``it will always be the case that if the treatment was given yesterday, then it is not given today, but will be given again tomorrow, and the patient will never recover";\\
\textbullet\ $(\sigma_2, 0)\vDash_{\PLTL} \F((\bigominus (T=1) \wedge \bigominus^2 (T=1)) \wedge \G(R=1))$, meaning that at $(\sigma_2, 0)$  ``eventually it will be the case that the treatment is given yesterday and two days ago, and from that moment forward the patient will be recovered".

\begin{definition}[Periodic model]\label{def:periodicmodel} A linear model $\sigma$ is \emph{ultimately periodic} if $\exists i,l>0$ such that $\sigma(k) = \sigma(k+l)$ for every $k\geq i$. We call the (possibly empty) finite sequence $\sigma(0), \dots, \sigma(i-1)$ the \emph{prefix} of $\sigma$ and $\sigma(i),\dots, \sigma(i+l)$ the \emph{loop} of $\sigma$ and say that $\sigma$ is of type $(i,l)$. 
\end{definition} 

An ultimately periodic model can be represented by a finite sequence $\vec{v}_1, \dots, (\vec{v})_{i+l}$ of the assignments of $\cV$. %The size of an ultimately periodic model $\sigma$, denoted as $|\sigma|$, with a prefix length $i$ and loop length $l$ is defined as $(i+l)\cdot |\cV|$.

\section{Temporal Interpretation of Causal Models}

% \alert{In this section we propose a temporal interpretation of causal models. Then, we demonstrate how non-recursive models behave in the proposed settings and discuss our modelling assumptions.}

In order to proceed we need to modify some of the definitions presented already. Firstly, we need to adjust the idea of contexts. Note that normally, the context $\vec{u}$ is understood as an assignment of all exogenous variables \cite{HalpernBook}. But in our setting, we want to consider contexts as a (time) series of such assignments describing how the values of exogenous variables evolve over time. 

\begin{definition}[Temporal context] A \emph{temporal} context $\overset{\twoheadrightarrow}{u}$ is an infinite sequence of complete assignments of $\cU$:
\[\overset{\twoheadrightarrow}{u}: \mathbb{N} \to \prod_{U\in \cU} \cR(U).\] We denote a particular time instance of $\overset{\twoheadrightarrow}{u}$ as $\overset{\twoheadrightarrow}{u}(n)$.
\end{definition}

We also need to adjust the definition of interventions. In our framework it is essential to specify not only which interventions take place, but also \textit{when}. We extend the notation to make interventions time sensitive and, instead of $Y\leftarrow y$, we use $Y(n)\leftarrow y$, where $n\in \mathbb{N}$, which means that we intervene on $Y$ with value $y$ at time step $n$. For multiple interventions, we use the notation $\vec{Y}(\vec{n})\leftarrow \vec{y} = (Y_1(n_1)\leftarrow y_1, \dots, Y_k(n_m)\leftarrow y_k)$. Note that we allow the same variable $Y'$ to occur in $\vec{Y}(\vec{n})\leftarrow \vec{y}$ multiple times, meaning that we can intervene on the same variable at multiple time moments. 

% \maksim{@all, Do you think the term 'continuous' is appropriate here? I am trying to say that some intervention can be 'prolonged' in time: $[Y(n)\leftarrow y, Y(n+1)\leftarrow y, \dots, Y(n+k)\leftarrow y]$, meaning that $Y$ keeps the same value $y$ for $k$ steps. But since the time is discreet, maybe we need another term?}

Given $\cM$ describing causal dependencies between the variables, and $\overset{\twoheadrightarrow}{u}$ describing how the values of exogenous variables evolve over time, we want to understand how the values of \textit{endogenous} variables evolve over time. We represent this evolution as an (infinite) sequence $\cC = (\vec{v_0}, \vec{v_1}, \dots)$ called a  \textit{computation}. 
%Note that $\cC$ is a linear model (\Cref{def:linear_model}). % We use the term computation, because the idea behind our temporal interpretation of structural equations is closely related to the process of computing the values of $\cV$ given $(\cM, \vec{u})$ for static SEM's presented in \Cref{sec:SEM's}. This idea was firstly proposed in \cite{Gladyshev_ECAI2023}.
First we define a \textit{call} to $\cF$. Let $\vec{u'}$ and $\vec{v'}$ be complete assignments of $\cU$ and $\cV$ respectively. We say that a \textit{call} to $\cF$ with $(\vec{u'}, \vec{v'})$ takes $(\vec{u'}, \vec{v'})$ and returns $v'' = \prod_{X\in \cV}\cF_X(\vec{u'}, \vec{v'})$. A \textit{computation} $\cC$ starts with a \textit{default} assignment $\vec{v}$ of $\cV$ (representing `initial' configuration of endogenous values)\footnote{The idea to use default assignments of $\cV$ was also discussed in the context of non-recursive models in \cite[Ch.2.7]{HalpernBook}} 
and evolves as a process of iterative calls to $\cF$, using values of $\cU\cup\cV$ from the previous step \footnote{The idea of a computation $\cC$ is adapted from \cite{Gladyshev_ECAI2023}.}. 

\begin{definition}[Computation]\label{def:computation} Given a tuple $(\cM, \overset{\twoheadrightarrow}{u}, \vec{v})$, a computation $\cC$ over $(\cM, \overset{\twoheadrightarrow}{u}, \vec{v})$ is a function mapping $\mathbb{N}$ to the complete assignments $\prod\limits_{V\in \cV} \cR(V)$ of endogenous variables, such that $\cC(\cM, \overset{\twoheadrightarrow}{u}, \vec{v})(0) := \vec{v}$ and for all $i>0$, $$\cC(\cM, \overset{\twoheadrightarrow}{u}, \vec{v})(i) := \prod_{X\in \cV}\cF_X\bigl(\overset{\twoheadrightarrow}{u}(i-1), \cC(\cM, \overset{\twoheadrightarrow}{u}, \vec{v})(i-1)\bigr)$$
We use the short notation $\cC(i)$ instead of $\cC(\cM, \overset{\twoheadrightarrow}{u}, \vec{v})(i)$ if $(\cM, \overset{\twoheadrightarrow}{u}, \vec{v})$ is clear from the context.
\end{definition}

Because any $\cC(i)$ is a vector of values whose coordinates are indexed by the elements of $\cV$, we write $\cC(i)|_X$ to refer to $X$'s value at the $i$'th step of $\cC$. An intervention $int=\vec{Y}(\vec{n})\leftarrow \vec{y}$  results in an updated computation $\cC^{\vec{Y}(\vec{n})\leftarrow \vec{y}}$, defined as follows. Given a default assignment $\vec{v}$, let $\vec{v}^{int}$ be an assignment of $\cV$, which agrees with $\vec{v}$ everywhere, except the variables $Y$, such that $Y(0)\leftarrow y$ occurs in  $\vec{Y}(\vec{n})\leftarrow \vec{y}$. The values of those variables in $\vec{v}^{int}$ are set according to $\vec{Y}(\vec{n})\leftarrow \vec{y}$.

% \begin{definition}[Updated Computation]\label{def:updatedcomputation} Given an intervention $\vec{Y}(\vec{n})\leftarrow \vec{y}$ and $(\cM, \overset{\twoheadrightarrow}{u}, \vec{v})$, an \emph{updated} computation $\cC^{\vec{Y}(\vec{n})\leftarrow \vec{y}}$ is defined as $\cC^{\vec{Y}(\vec{n})\leftarrow \vec{y}}(0)=\vec{v^*}$, and for all $i>0, X\in \cV$: if $X(i)\leftarrow x' \in \vec{Y}(\vec{n})\leftarrow \vec{y}$, then $(X=x')\in \cC^{\vec{Y}(\vec{n})\leftarrow \vec{y}}(i)$; otherwise $(X=x)\in \cC^{\vec{Y}(\vec{n})\leftarrow \vec{y}}(i)$ for $x = \cF_X(\overset{\twoheadrightarrow}{u}(i-1), \cC^{\vec{Y}(\vec{n})\leftarrow \vec{y}}(i-1))$.
% \end{definition}

\begin{definition}[Updated Computation]\label{def:updatedcomputation} Given an intervention $int=\vec{Y}(\vec{n})\leftarrow \vec{y}$ and $(\cM, \overset{\twoheadrightarrow}{u}, \vec{v})$, an \emph{updated} computation $\cC^{int}$ is defined as $\cC^{int}(0)=\vec{v}^{int}$, and $\forall i>0\forall X\in \cV$: 
$$\cC^{int}(i)|_X=\begin{cases}
			x', \text{ if } X(i)\gets x' \in \vec{Y}(\vec{n})\leftarrow \vec{y}\\
            \cF_X(\overset{\twoheadrightarrow}{u}(i-1), \cC^{int}(i-1)), \text{ otherwise}
		 \end{cases}$$
\end{definition}

%\maksim{@Dragan, could you check the new Def?}

Simply speaking, an updated computation $\cC^{int}$ replaces the values of variables from $int$ on the corresponding steps.

Recall \Cref{exm:rocks} and let $\overset{\twoheadrightarrow}{u} = (00,10,00,01,00 \dots)$, so the generated computation for $\overset{\twoheadrightarrow}{u}$ and $\vec{v} = 000$ (here we write 000 instead of (ST=0, BT=0, BS=0)) is $\cC = (000, 000, 100, 001,010,001, \dots)$. Suzy throws (ST=1) at step 2 and Billy (BT=1) at step 4 (we start counting from 0). Then, we can say that the $\LTL$ formula $\bigcirc^2(ST=1) \wedge \bigcirc^3(BS=1) \wedge \bigcirc^4(BT=1)$ is true at $\cC(0)$. And counterfactually, $\bigcirc(BS=1)$ is true in $\cC^{BT(0)\leftarrow 1(0)}$.\footnote{You may notice that in this computation the bottle shatters at step 3 due to Suzy's throw, then BS=0 happens again at step 4 because both (ST=0, BT=0) hold at $\cC(3)$, then BS=1 holds again at $\cC(5)$ due to Billy's throw. This is an artifact of structural equations being defined in a specific way. We discuss how this can be fixed below.}

As can be seen, our approach allows us to merge causal time-sensitive interventions with the machinery of $\PLTL$. We call this logic \textit{Causal LTL with Past} ($\mathsf{CPLTL}$).

\begin{definition}[Syntax of $\mathsf{CPLTL}$]\label{def:syntaxCPLTL}  The grammar of $\CPLTL$ is defined as follows:
   \[\varphi::=  [\vec{Y}(\vec{n})\leftarrow \vec{y}]\psi \mid \neg\varphi\mid \varphi\wedge\varphi,\]

    % \psi::= (X=x) \mid \neg \psi \mid \psi\wedge \psi \mid \bigcirc\psi \mid \psi\U\varphi \mid \bigominus\psi\mid \psi\sS\psi, 

    where $\psi\in \PLTL$ (\Cref{def:syntaxPLTL}) and $\vec{Y}(\vec{n})\leftarrow \vec{y} = (Y_1(n_1)\leftarrow y_1, \dots, Y_k(n_m)\leftarrow y_k)$, such that $Y_i\in \cV$, $y_i\in \cR(Y_i)$, $n_i\in \mathbb{N}$ and for any $Y_i(n_i)\leftarrow y_i, Y_j(n_j)\leftarrow y_j \in \vec{Y}(\vec{n})\leftarrow \vec{y}$, $Y_i=Y_j$ implies $n_i\neq n_j$. $\vec{Y}(\vec{n})\leftarrow \vec{y}$ may be empty, in this case we write $\psi$ instead of $[]\psi$. We use the same abbreviations for boolean connectives and temporal operators as in $\PLTL$. 
\end{definition}

In contrast to static causal reasoning, where formulas are evaluated wrt a causal setting $(\cM, \vec{u})$, to evaluate $\CPLTL$ formulas, we need to know a causal model $\cM$, a temporal context $\overset{\twoheadrightarrow}{u}$, a default assignment $\vec{v}$, and a time moment $t$. We call $(\cM, \overset{\twoheadrightarrow}{u}, \vec{v})$ a causal scenario. Note that $(\cM, \overset{\twoheadrightarrow}{u}, \vec{v})$ together with an intervention $\vec{Y}(\vec{n})\leftarrow \vec{y}$ produces a computation $\cC^{\vec{Y}(\vec{n})\leftarrow \vec{y}}$ according to \Cref{def:updatedcomputation}, which is a linear model in the sense of \Cref{def:linear_model}, used to define the semantics of the $\PLTL$ fragment.

\begin{definition}[Semantics of $\CPLTL$]\label{def:semanticsCPLTL}
   Given a causal scenario $(\cM, \overset{\twoheadrightarrow}{u}, \vec{v})$, $t\in\mathbb{N}$ and $\varphi\in \CPLTL$ we define truth relation $(\cM, \overset{\twoheadrightarrow}{u}, \vec{v}),t\vDash \varphi$ inductively as follows:\\
   \\
    $(\cM, \overset{\twoheadrightarrow}{u}, \vec{v}), t\vDash [\vec{Y}(\vec{n})\leftarrow \vec{y}]\psi$ iff  $(\cC^{\vec{Y}(\vec{n})\leftarrow \vec{y}}, t)\vDash_{\PLTL} \psi$,\\
    where $\vDash_{\PLTL}$ is introduced in \Cref{def:semanticsPLTL};\\
     $(\cM, \overset{\twoheadrightarrow}{u}, \vec{v}), t\vDash \neg \varphi$ iff  $(\cM, \overset{\twoheadrightarrow}{u}, \vec{v}), t\nvDash \varphi$;\\
     $(\cM, \overset{\twoheadrightarrow}{u}, \vec{v}), t\vDash \varphi\wedge \chi$ iff  $(\cM, \overset{\twoheadrightarrow}{u}, \vec{v}), t\vDash \varphi \text{\&}  (\cM, \overset{\twoheadrightarrow}{u}, \vec{v}), t\vDash \chi$.
\end{definition}

\paragraph{Non-Recursiveness}

The temporal approach to SEM's proposed above not only allows to deal with non-recursive models without additional technical adjustments, but also often provides more elegant ways to describe the desired temporal behaviour of the system. 

Consider \Cref{exm:treatment} again. To model this scenario, we want our model $\cM$ to contain $\cV = \{T, R\}$ (for Treatment and Recovery), such that (1) $R=1$ once the treatment is given twice in a row, and (2) once $R=1$, it remains so. Let variable $U (\cR(U)=\{0, 1\})$ represent whether the treatment is given in a given moment. And let $\cR(T)= \{0, 1\}, \cR(R)= \{0, \frac{1}{2}, 1\}$, where T=1 means the treatment is given. We want to define our structural equations in such a way that R=0 if the treatment is not given on the previous step, R=$\frac{1}{2}$ if the treatment was given once, and R=1 if the treatment was given twice in a row. Additionally, we require that if the patient is recovered, he must remain so.

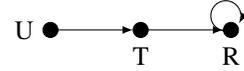
\begin{figure}[h]
\centering
\begin{tikzpicture}
  \node[dot] [label=left:{U}] (U){};
    \node[dot] [label=below:{T}] (T) [right=of U]{};
    \node[dot] [label=below:{R}] (R) [right=of T]{};

    \draw[-latex](U) -- (T);
    \draw[-latex](T) -- (R);
    \draw[-latex] (R) arc(290:-30:.2);
\end{tikzpicture}
\caption{Non-recursive representation of \Cref{exm:treatment}.}
\end{figure}

The desired behaviour of the system may be achieved if we define $\cF_{R}$ as follows. $R:=0$ if $(T=0 \wedge R\neq 1)$; $R:=\frac{1}{2}$ if $(T=1\wedge R=0)$; $R:=1$ if $(T=1\wedge R=\frac{1}{2}) \vee R=1$. 
  This model is clearly non-recursive, because $\cF_R$ depends of $R$. However, note that under temporal interpretation every edge in the dependency graph takes (at least) 1 time interval to proceed. It is a feature of computation $\cC$, which performs consecutive calls to $\cF$, where values of $\cV$ at any step $i$ depend on values of $(\cU\cup\cV)$ on the \textit{previous} step. So, $R:=1$ if $((T=1 \wedge R=\frac{1}{2})\vee R=1)$ means that ``R=1 is true \textit{now} if on the \textit{previous} step both T=1 and R=$\frac{1}{2}$ were true, or R=1 was true.''
 
  Given a temporal context, e.g. $\overset{\twoheadrightarrow}{u_1} = (0,1,0,1,1,0\dots)$ and a default assignment $\vec{v} = 00$ (we write 00 instead of (T=0,R=0)), $(\cM, \overset{\twoheadrightarrow}{u_1}, \vec{v})$ generates a computation $\cC = (00,00,10,0\frac{1}{2},10,1\frac{1}{2},01,01,\dots)$, in which $\overset{\twoheadrightarrow}{u_1}(1)$=1 triggers (T=1 at t=2), which triggers (R=$\frac{1}{2}$ at t=3). But since $\overset{\twoheadrightarrow}{u_1}(2)$=0, T=0 becomes true at t=3, leading to (R=0 at t=4). Later, at step 4, T=1 happens again, triggering (R=$\frac{1}{2}$ at t=5). Since both T=1 and  R=$\frac{1}{2}$ are true at t=5, R=1 triggers at t=6. From this moment, R=1 remains true at any t=i, because R=1 holds at i-1.
This corresponds to the temporal behavior we wanted to achieve in \Cref{exm:treatment}. Our $\CPLTL$ language allows to formulate such statements as
$(\cM, \overset{\twoheadrightarrow}{u}, \vec{v}), 6\vDash (R=1) \wedge \mathsf{H}\neg(R=1) \wedge [T(0)\leftarrow 1]\bigominus^3\G(R=1)$ meaning that at step 6 it is true that: (1) (R=1); (2) (R=1) has never been true before ($\mathsf{H}\neg(R=1)$); and (3) if the intervention $T\leftarrow 1$ was performed at step 0, R=1 would have been true for 3 time steps already (and would remain so forever). Note also that the same computation could be generated for the trivial context $\overset{\twoheadrightarrow}{u_2} = (0,0,\dots)$ and an intervention $int' = (T(2)\gets 1, T(4)\gets 1, T(5)\gets 1)$, so $\cC(\cM, \overset{\twoheadrightarrow}{u_1}, \vec{v}) = \cC^{int'}(\cM, \overset{\twoheadrightarrow}{u_2}, \vec{v})$.

In many cases non-recursive causal models are the only way to represent mutually dependent variables and feedback loops processes, which are necessary to model many interesting phenomena. However, most of the literature on actual causality is restricted to recursive models because (in static settings) non-recursiveness may create serious technical difficulties, leading to non-uniqueness of the solution of structural equations with no clear way to choose the 'correct' one \cite[Ch. 2.7]{HalpernBook}. We argue that non-recursive models do not create any technical difficulty under the temporal interpretation of SEM's. But also sometimes provide fruitful modelling tools, as we will later see. 

 % \maksim{@Mehdi, I think the next section prepares the reader for a discussion on temporal equivalence, but its main aim is still to discuss assumptions about models that we make. And examples support this discussion. But please let me know if you want to rename/modify it.}

\paragraph{Modelling Assumptions} Here we list our modelling assumptions.
%, which the proposed interpretation is suitable for. 
First of all, 
in our settings the time is discrete. This is a standard assumption for $\LTL$-style temporal logics. 
% Natasha 15 August: this is not an assumption
%Secondly, we do not restrict our attention to recursive models. Thirdly, 
We also assume that the temporal context $\overset{\twoheadrightarrow}{u}$ represents a time series of exogenous changes given to us as an input. In this time series equal intervals between indexes correspond to equal time intervals. And similarly, equal time intervals correspond to equally spaced indices of the computation $\cC$. Given $\overset{\twoheadrightarrow}{u}$ time series, we want our model to return the correct time series of $\cV$ values that (temporally) correspond to the behaviour of the phenomena of our interest. So, our framework requires causal models to contain correct temporal information, which affects the way we design them. 

To illustrate this, let us revisit \Cref{exm:rocks}. Assume we know that Suzy's throws are consistently faster that Billy's. Let us say it takes $n$ time steps (e.g. seconds)  for Suzy's rock to reach the bottle, and $k$ for Billy's, where $n<k$. So, whenever Suzy decides to throw the rock at time $t_S$ and Billy at $t_B$, the Suzy's rock will reach the bottle (if it is still there) at time $t_S+n$ and Billy's at time $t_B+k$. We want our model to predict when the bottle will be shattered, given a temporal context $\overset{\twoheadrightarrow}{u}$. So, our model must contain the information about `delays' between the Suzy's (or Billy's) throw and the bottle shattering. One way to achieve this, is to add `chains' of hidden (i.e. dummy) variables in the model (\Cref{fig:temorap_rock} (a)).

\begin{figure}[ht]
\centering
    \begin{tikzpicture}
    \node[dot] [label=above:{BS}] (BS) {};
    \node[dot][label=above:{$X_{n-1}$}] (Xn) [above left=.25cm and .5cm of BS] {};
    \node[dot] [label=below:{$Y_{k-1}$}] (Yk) [below left=.25cm and .5cm of BS] {};
    \node[dot] [label=left:{ST}] (ST) [left=2.3cm of Xn] {};
    \node[dot] [label=left:{BT}] (BT)  [left=2.3cm of Yk]  {};
    \node[dot] (X1) [label=above:{$X_1$}] [right=.5cm of ST] {};
    \node (X2) [right= .5cm of X1] {$\cdots$};
     \node[dot] (Y1) [label=below:{$Y_1$}] [right=.5cm of BT] {};
    \node (Y2) [right=.5cm of Y1] {$\cdots$};

%Lines
    \draw[-latex] (ST) -- (X1);
    \draw[-latex] (BT) -- (Y1);
    \draw[-latex] (X1) -- (X2);
    \draw[-latex] (X2) -- (Xn);
    \draw[-latex] (Xn) -- (BS);
    \draw[-latex] (Y1) -- (Y2);
    \draw[-latex] (Y2) -- (Yk);
    \draw[-latex] (Yk) -- (BS);
    \draw[-latex] (BS) arc(90:-180:.2);

\node (a) [below right=.3cm and .2cm of Y1] {(a)};

    \node[dot] [label=left:{ST}] (ST2) [right= 2cm of Xn] {};
\node[dot] [label=left:{BT}] (BT2) [below=.5cm of ST2] {};
\node[dot] [label=above:{X}] (X) [right=.5cm of ST2] {};
\node[dot] [label={[label distance=.1cm]300:Y}] (Y) [right=.5cm of BT2] {};
\node[dot] [label=above:{BS}] (BS2) [below right= .25cm and .7cm of X] {};
    \draw[-latex] (ST2) -- (X);
    \draw[-latex] (BT2) -- (Y);
    \draw[-latex] (X) -- (BS2);
    \draw[-latex] (Y) -- (BS2);
    \draw[-latex] (X) arc(70:-200:.2);
    \draw[-latex] (Y) arc(90:-200:.2);
    \draw[-latex] (BS2) arc(90:-180:.2);

\node (b) [below=.3cm of Y] {(b)};
\end{tikzpicture}
    \caption{(a) 'Long' and (b)'Chain-free' models.}
\label{fig:temorap_rock}
\end{figure}
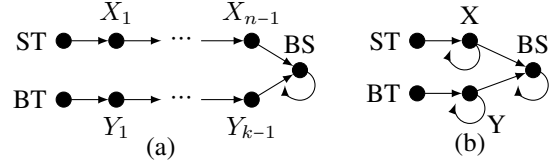

 Firstly, we assume that if BS=1 at some step, then it should remain so, because once the bottle is shattered it obviously remains so. This creates a reflexive arrow in the dependency graph. If they throw simultaneously at step $i$, then the bottle shatters at step $(i+n)$, because $n<k$. But now we can model situations when they decide to throw a rock at different time. So, if Suzy decides to throw at $t_S$ and Billy at $t_B$, then the bottle will be shattered at $t^*:=\min((t_S+n), (t_B+k))$, i.e. for $\overset{\twoheadrightarrow}{u}=(00,00,\dots)$, $(\cM, \overset{\twoheadrightarrow}{u}, \vec{v}), 0\vDash [ST(t_S)\gets 1, BT(t_B)\gets 1]\bigcirc^{t^*} (BS=1)$.

Such a representation is not compact, and non-recursive models provide us a better way to represent this example. Instead of adding a long chain of dependencies to capture time intervals between events and their effects, we can add a single variable to abbreviate each chain (\Cref{fig:temorap_rock} (b)).
Note, however, that it is not enough to specify the range of these new variables, $X$ and $Y$, as $\cR(X)=\{0, \dots, n-1\}$ and $\cR(Y)=\{0, \dots, k-1\}$, where each value `emulates' position of the rock on the original chain. This is because nothing prevents multiple variables in the corresponding chains $X_1\dots X_{n-1}$ and $Y_1\dots Y_{k-1}$ to have value 1 at the same time moment. This situation can be interpreted as multiple rocks thrown at different moments. To properly encode the temporal behaviour of the original model using a `chain-free' model, the range of $X$ and $Y$ must contain all binary strings of the length $n-1$ and $k-1$ respectively. In other words, the new values must represent not only in which position the rock is at a given time moment, but also how many rocks there are. The equations then can be defined straightforwardly, and we omit the formal description due to lack of space.
% specified in such a way that $BS:=1$ whenever the last digit in $X$'s or $Y$'s value is 1 or $BS=1$ already; $X$'s current value at the position $i=0$ has the previous value of $ST$ and on the position $i$ (for $0<i\leq n_1$)\\
% $X:=1$ if $ST=1\wedge X=0$; for $1<i\leq n-1$ $X:= i$ if $X=(i-1)$; $X:=0$ otherwise; and analogously for $Y$. Now $BS:=1$ if (($X=n-1 \vee Y=k-1$) $\vee$ $BS=1$);
This construction provides a more compact graphical representation of the model, by reducing long chains of dependencies. It is easy to verify that as long as we are interested only in the variables $\{ST, BT, BS\}$, these two models behave identically wrt any context and time moment. We discuss the notion of  temporal equivalence in detail in the next section.

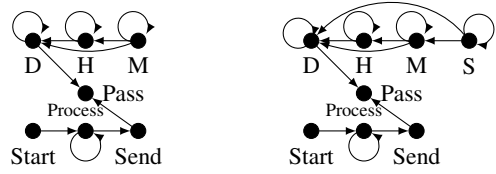
\begin{figure}[t]

\begin{tikzpicture}
\centering
   \node[dot] (D) [label=below:{\small{D}}] {};
   \node[dot] [right=.5cm of D] (H) [label=below:{\small{H}}] {};
   \node[dot] [right=.5cm of H] (M) [label=below:{\small{M}}] {};
    \node[dot] [below=1cm of D] (Start) [label=below:{\small{Start}}] {};
    \node[dot] [right=.5cm of Start] (Process) {};
    \node [above right=.001cm and -.001cm of Start] {\scriptsize{Process}};
\node[dot] [right=.5cm of Process] (Send) [label=below:{\small{Send}}] {};
\node[dot] [below=.5cm of H] (Pass) [label=right:{\small{Pass}}] {};
\draw[-latex] (H) -- (D);
\draw[-latex] (M) -- (H);
\draw[-latex] (Start) -- (Process);
\draw[-latex] (Process) -- (Send);
\draw[-latex] (D) arc(290:-30:.2);
\draw[-latex] (H) arc(290:-30:.2);
\draw[-latex] (M) arc(290:-30:.2);
\draw[-latex] (Process) arc(90:420:.2);
\draw[-latex] (M) to[out=200,in=-15] (D);
\draw[-latex] (D) -- (Pass);
\draw[-latex] (Send) -- (Pass);

\node (a) [below=.75cm of Process] {(a) Time scale = 1 min.};

   \node[dot] (D1) [label=below:{\small{D}}] [right=3.5cm of D]{};
   \node[dot] [right=.5cm of D1] (H1) [label=below:{\small{H}}] {};
   \node[dot] [right=.5cm of H1] (M1) [label=below:{\small{M}}] {};
    \node[dot] [below=1cm of D1] (Start1) [label=below:{\small{Start}}] {};
    \node[dot] [right=.5cm of Start1] (Process1) {};
     \node [above right=.001cm and -.001cm of Start1] {\scriptsize{Process}};
\node[dot] [right=.5cm of Process1] (Send1) [label=below:{\small{Send}}] {};

\node[dot] [below=.5cm of H1] (Pass1) [label=right:{\small{Pass}}] {};

\draw[-latex] (H1) -- (D1);
\draw[-latex] (M1) -- (H1);
\draw[-latex] (Start1) -- (Process1);
\draw[-latex] (Process1) -- (Send1);
\draw[-latex] (D1) arc(330:-15:.2);
\draw[-latex] (H1) arc(290:-30:.2);
\draw[-latex] (M1) arc(290:-30:.2);
\draw[-latex] (Process1) arc(90:420:.2);
\draw[-latex] (M1) to[out=200,in=-15] (D1);
\draw[-latex] (D1) -- (Pass1);
\draw[-latex] (Send1) -- (Pass1);

\node[dot] [right=.5cm of M1] (S) [label=below:{\small{S}}] {};
\draw[-latex] (S) arc(230:-100:.2);
\draw[-latex] (S) -- (M1);
\draw[-latex] (S) to[out=130,in=45] (D1);

\node (a) [below=.75cm of Send1] {(b) Time scale = 1 sec.};
   \end{tikzpicture}
\caption{Dependency graphs for \Cref{exm:deadline}.}
\label{fig:deadline}
\end{figure}
%\end{example}

Our models must adequately represent temporal behaviour of a system. This, in turn, requires a clear temporal semantics of each `tick' of the model. To illustrate the problem, we present our final example.

\begin{example}[Deadline]\label{exm:deadline} Assume the agent has a deadline to perform some task. Let variables $D, H, M$ with $\cR(D)$=$\{Mon, \dots, Sun\}), \cR(H)$=$\{0, \dots, 23\}$ and $\cR(M)$=$\{0, \dots, 59\}$ represent days, hours and minutes respectively. At any moment the agent may decide to Start the task (Start=1). It takes 8 hours for the agent to process the task, so the range of Process variable contains $8*60=480$ values, representing minutes. Once the task is completed, it is sent to the server (Send=1). If this happens not later than Friday (D$\neq$Sat$\wedge$D$\neq$Sun), then the task is passed (Pass = 1); otherwise Pass = 0.
\end{example}

The model shown in Figure \ref{fig:deadline} (a) obviously allows us to reason about the agent's decisions, and whether the deadline is met. However the interesting point is that there are clocks embedded in the model, and is obvious from the model itself which time interval we denote as 1 computational step. We can also easily change our time scale by adding a variable S for seconds and modifying the structural equation in obvious way, see Figure \ref{fig:deadline} (b). Note, however, that the two models are not equivalent (in the sense of the previous example), because now it takes 60 times more steps for any of the variable D, H, M to change value. At the same time, it allows us to model various sub-processes with higher accuracy. So, in our framework it is crucial to understand how 1 computational step of a model is interpreted in terms of real-world time, i.e., to understand which clocks are supposed to tick along with a model.

\section{Temporal Equivalence for Causal Models}

When dealing with structural equation modelling, we usually have many alternative causal models that describe the same underlying process. These models may have different sets of variables and describe causal dependencies in different ways. Moreover, at some point we may expand the set of variables in a model or reconsider some dependencies due to new discoveries. The only requirement ensuring that different models talk about the same process is that the models share some set of common variables. In such settings, it is crucial to have an adequate notion of equivalence between models to guarantee that different models correctly represent causal (and temporal) properties of some process with respect to the variables of interest  \cite{Beckers2021eq}.
%As \cite{Beckers2021eq} noted: \textit{``no matter what further
%variables we discover to be causally connected to the
%variables of interest, and no matter what variables that we
%are not interested in are being marginalized away, as long as
%the resulting causal model is equivalent to the original one
%the conclusions about their common variables remain identical"}. 
Since previous work has focused on static interpretation of SEMs (and so usually applicable only to recursive models), in this section we discuss how model equivalence can be treated in our framework.  

Following \cite{Beckers2021eq}, we assume that, two models $\cM_1$ and $\cM_2$ share the same exogenous variables ($\cU_1=\cU_2$) and a set of \textit{observable} variables $\cO\subseteq (\cV_1\cap \cV_2)$. So, we can only observe the values of and perform interventions on  $\cO$.  

\begin{definition}[Equivalent Computations]
    Consider two computations $\cC_1$ and $\cC_2$ sharing some set of variables $\cO$. We call $\cC_1$ and $\cC_2$ \emph{temporally equivalent} wrt $\cO$ if\\ $\forall X\in \cO, \forall i>0: (X=x)\in \cC_1(i)$ iff $(X=x)\in \cC_2(i)$.
\end{definition}

In other words, at any step equivalent computations agree on the values of all variables in $\cO$. %Now, we generalize this notion to models.

\begin{definition}[Model Equivalence]\label{def:equiv:model} Two models $\cM_1$ and $\cM_2$ are temporally equivalent wrt $\cO$ if for any intervention $\vec{Y}(\vec{n})\leftarrow \vec{y}$, where $\vec{Y}\subseteq \cO$ and for any $(\overset{\twoheadrightarrow}{u}, \vec{v_1})$ there exists $\vec{v_2}$, such that computations for $(\cM_1^{\vec{Y}(\vec{n})\leftarrow \vec{y}}, \overset{\twoheadrightarrow}{u}, \vec{v_1})$ and $(\cM^{\vec{Y}(\vec{n})\leftarrow \vec{y}}_2, \overset{\twoheadrightarrow}{u}, \vec{v_2})$ are equivalent, and vice versa. 
\end{definition}

Temporal equivalence of causal models guarantees that no matter what intervention $\vec{Y}(\vec{n})\leftarrow \vec{y}$ we use, there is no difference in $\cC_1^{\vec{Y}(\vec{n})\leftarrow \vec{y}}$ and $\cC_2^{\vec{Y}(\vec{n})\leftarrow \vec{y}}$ in how the endogenous changes in $\cO$ proceed with exogenous changes $\overset{\twoheadrightarrow}{u}$.

\begin{observation} 
    Models $\cM_a$ and $\cM_b$ in \Cref{fig:temorap_rock} are temporally equivalent for $\cO=\{ST, BT, BS\}$.
\end{observation}

It is easy to check, that whatever the context $\overset{\twoheadrightarrow}{u}$, default assignment $\vec{v_1}$ and an intervention $\vec{Y}(\vec{n})\leftarrow \vec{y}$ for $\vec{Y}\subseteq \cO$ are, there exists a default assignment $\vec{v_2}$, such that $(\cM_a^{\vec{Y}(\vec{n})\leftarrow \vec{y}}, \overset{\twoheadrightarrow}{u}, \vec{v_1})$ and  $(\cM_b^{\vec{Y}(\vec{n})\leftarrow \vec{y}}, \overset{\twoheadrightarrow}{u}, \vec{v_2})$ generate $\cO$-equivalent computations. However, this notion of equivalence is too strong to capture the similarities in \Cref{exm:deadline}.

\begin{observation} 
    Models $\cM_a$ and $\cM_b$ from \Cref{exm:deadline} are \textbf{not} temporally equivalent wrt $\cO=\{Start, Pass\}$.
\end{observation}

We therefore need a more general notion of equivalence. Note that the models in \Cref{exm:deadline} describe identical processes, but on a different time scale. To capture this aspect, we introduce the notion of \textit{rescalable} equivalence.

\begin{definition}[Rescalably Equivalent Computations] $\cC_2$ is \emph{rescalably} equivalent to $\cC_1$ wrt $\cO$ (with a coefficient $k\in \mathbb{N}$) if $\forall X\in \cO, \forall i>0: (X=x)\in \cC_1(i)$ iff $(X=x)\in \cC_2(i\cdot k)$.
\end{definition}

Informally, $k$ demonstrates how many ticks of $\cM_2$ are needed to emulate one tick of $\cM_1$. Given an intervention $\vec{Y}(\vec{n})\leftarrow\vec{y}$ and a coefficient $k$, $\vec{Y}(\vec{n}^k)\leftarrow\vec{y}$ denotes an intervention, in which all indexes from $\vec{n}$ are multiplied by $k$.

\begin{definition}[Rescalably Equivalent Models] A model $\cM_2$ is \emph{rescalably} equivalent (with a coefficient $k$) to $\cM_1$ wrt $\cO$ if for any intervention $\vec{Y}(\vec{n})\leftarrow\vec{y}$ ($\vec{Y}\subseteq \cO$) and for any $(\overset{\twoheadrightarrow}{u}, \vec{v_1})$ there exists $\vec{v_2}$, such the computation for $(\cM_2^{\vec{Y}(\vec{n^k})\leftarrow\vec{y}}, \overset{\twoheadrightarrow}{u}, \vec{v_2})$  is rescalably equivalent (with a coefficient $k$) to the computation for $(\cM_1^{\vec{Y}(\vec{n})\leftarrow\vec{y}}, \overset{\twoheadrightarrow}{u}, \vec{v_1})$.
\end{definition}

\begin{observation}
    Model $\cM_b$ in \Cref{exm:deadline} is rescalably equivalent to $\cM_a$ wrt $\cO=\{Start, Pass\}$ with $k=60$.
\end{observation} 

% Natasha 15 Aug: slightly shortened
%Even though \cite{Beckers2021eq} claims that \textit{``If the 'numerical' properties [e.g. the number of intermediate variable along an edge or the number of distinct paths
%from one variable to another] were essential, then it seems impossible
%to come up with any useful notion of equivalence"} , we believe this is exactly what happens in our approach. Because in temporal settings it is essential to consider not only \textit{what} happened, but also how much time it took to happen.

We disagree with \citet{Beckers2021eq} that it is impossible to come up with a useful notion of equivalence that takes into account `numerical' properties.  We believe that we have proposed such a notion, and it is useful because in temporal settings it is essential to consider not only \textit{what} happened, but also how much time it took to happen.

\section{Model-Checking}

% Given a model $\cM$ and a temporal context $\overset{\twoheadrightarrow}{u}$, we would like to have a formal verification technique to check if some causal temporal formula $\varphi$ is true in this model. Here 

In this section, we study the model-checking complexity of $\CPLTL$.

\begin{definition}[Model-checking] The $\CPLTL$ \emph{model-checking problem} is, given $(\cM, \overset{\twoheadrightarrow}{u}, \vec{v}, t)$ and a $\CPLTL$ formula $\varphi$, to decide whether $(\cM, \overset{\twoheadrightarrow}{u}, \vec{v}), t\vDash \varphi$.
\end{definition}

Note that the input to the problem is not necessarily finite or finitely presentable, since only $\cM, \vec{v}$ and $t$ are finite objects. To ensure that $\overset{\twoheadrightarrow}{u}$ is finitely presentable, analogously to \Cref{def:periodicmodel}, we require that $\overset{\twoheadrightarrow}{u}$ is \textit{ultimately periodic} of type $(n, m)$, i.e. there exist $n,m>0$ such that $\overset{\twoheadrightarrow}{u}(k)=\overset{\twoheadrightarrow}{u}(k+m)$ for every $k\geq n$. We say that $\overset{\twoheadrightarrow}{u}[0, n-1]$ is the prefix and $\overset{\twoheadrightarrow}{u}[n, n+m]$ is the loop of $\overset{\twoheadrightarrow}{u}$. 

\begin{definition}[Model Size] The size of a periodic temporal context $\overset{\twoheadrightarrow}{u}$ of type $(n, m)$ is  $$|\overset{\twoheadrightarrow}{u}| = \sum\limits_{0\leq j\leq (n+m)}|\overset{\twoheadrightarrow}{u}(j)|.$$
The size of a model $\cM$, denoted $|\cM|$, is $$|\cM|:= |\cU|+|\cV| + |\cR| + |F|,$$
where $|\cR| = \Sigma_{X\in (\cU\cup \cV)}|\cR(X)|$ and
$|\cF|$ is the cardinality of the set of tuples in the extensional definition 
%(\Cref{def:causalmodel}) 
of $\cF$ (which is $|\cV| \cdot \prod_{X\in (\cU\cup \cV)}|\cR(X)|)$). %\alert{Note that this is a finite set as the range of variables is finite}
$|\vec{v}|= |\cV|$. We assume that all numbers are written in unary, in particular $|t| = t$.
%and let $t$ be given in binary notation, so $|t|=log_2t+1$. 
Then, 
$|(\cM, \overset{\twoheadrightarrow}{u}, \vec{v}, t)|= |\cM|+|\overset{\twoheadrightarrow}{u}|+|\vec{v}|+|t|$.
\end{definition}

% Natasha 15 Aug: commetend next paragraph  out since reintroduced in the proof of Theorem 1
%Given an $\CPLTL$ formula $\varphi$, let $Sub(\varphi)$ be a set of subformulas of $\varphi$, $Main(\varphi) = \{\theta\in Sub(\varphi)\mid \theta \text{ is of the form } [\vec{Y}(\vec{n}\leftarrow \vec{y})]\psi\}$ be a set of \textit{main} subformulas of $\varphi$, and  $Int(\varphi)$ be a set of all interventions $\vec{Y}(\vec{n})\leftarrow \vec{y}$ from $Main(\varphi)$. Finally, let $n_{max}^{Int(\varphi)}$ be the largest time index from $Int(\varphi)$ (i.e. the time of the latest intervention occurring $\varphi$). 

%Natasha 14 Aug: simplified
%A size of $\varphi$, denoted $|\varphi|$, is %defined as $|\varphi|:= %|Sub(\varphi)|+n_{max}^{Int(\varphi)}$. 

The size of $\varphi$, denoted $|\varphi|$, is the number of symbols in $\varphi$, assuming that numbers are written in unary.\footnote{Unary encoding gives more intuitive results, e.g., traversing a path up to position $n$ is linear rather than exponential in $|n|$.}

Informally, our approach is as follows. First,  we show that computation $\cC^{int}$ for $(\cM, \overset{\twoheadrightarrow}{u}, \vec{v})$ and an intervention $int$ from $\varphi$ is ultimately periodic of type $(x,y)$ for some $x,y\in \mathbb{N}$. Then, we use a $\PLTL$ \textit{path} model-checker \cite{MARKEY200287} to verify $(\cC^{int}, t)\vDash_{\PLTL}\psi$ for each intervention subformula of $\varphi$ of the form $[int]\psi$ 
%. Finally, because $\varphi$ is a boolean combination of such subformulas, it remains to run the above mentioned procedure for each intervention subformula of $\varphi$ 
and compute the value of their boolean combination.

\begin{algorithm}[t!]
\caption{Computing $\cC^{\vec{Y}(\vec{n})\leftarrow \vec{y}}$ of ($M, \overset{\twoheadrightarrow}{u}, \vec{v}, \vec{Y}(\vec{n})\leftarrow \vec{y}$), type of $\overset{\twoheadrightarrow}{u}$ is $(n,m)$ }\label{alg:compute} 
\small
\begin{algorithmic}[1] 
\Procedure{\textsc{periodic-comp}}{$M, \overset{\twoheadrightarrow}{u}, \vec{v}, \vec{Y}(\vec{n})\leftarrow \vec{y}$}
%\State $int \gets \vec{Y}(\vec{n})\leftarrow \vec{y}$
\State $n^{int}\gets \max(n_1, \dots, n_k \in \vec{n})$
\State $n^* \gets max(n,n^{int})$
\State $C(0) \gets \{X = x' \mid X(0) \gets x' \in int\}$
\State $C(0)\gets C(0)\cup \{ X=x'' \mid X = x'' \in \vec{v}\ \wedge\ $
\StatexIndent[8.5] $\nexists x (X(0) \gets x \in int) \}$
 \For{$i\in [1, n^*]$}
\If{$i < n^{int}$}
\State $C(i) \gets \{ X = x' \mid X(i) \gets x' \in int \}$
\State $C(i) \gets C(i) \cup \{ X = \cF_X(\overset{\twoheadrightarrow}{u}(i-1), C(i-1)) \mid$
\StatexIndent[7.75] $\nexists x (X(i) \gets x \in int)\} $
\EndIf
\State $C(i) \gets \{ X = \cF_X(\overset{\twoheadrightarrow}{u}(i-1), C(i-1)  \mid X \in \cV\} $
\EndFor
% got to either n or n^int; compute position p and start looking for a cycle
\Repeat
\State $i \gets i + 1$
\State $C(i) \gets \{ X = \cF_X(\overset{\twoheadrightarrow}{u}(i-1), C(i-1)  \mid X \in \cV\} $
\Until{$C(i) = C(j) \wedge (i \mod m) = (j \mod m)$ 
\StatexIndent[2.5] for some $n^* \leq j < i$ }
\State $y \gets i - j$
\State $x \gets i - y$
\State{$\cC^{int} \gets C$}
%\State{$\cC^{\vec{Y}(\vec{n})\leftarrow \vec{y}} \gets C$}
\EndProcedure
\end{algorithmic}
\end{algorithm}
%\end{breakablealgorithm}

Algorithm \ref{alg:compute} first computes the prefix of the computation by applying all the interventions (lines 7--9) and, if necessary, continues the computation until the end of the prefix of $\overset{\twoheadrightarrow}{u}$ (line 10). We then start searching for a cycle (lines 11-14). Note that, to find a cycle, it is not sufficient to find $i, j$ such that $\cC(i)=\cC(j)$ and $\overset{\twoheadrightarrow}{u}(i) =
\overset{\twoheadrightarrow}{u}(j)$, as $\overset{\twoheadrightarrow}{u}(i)$ can be the same at different points in the loop of $\overset{\twoheadrightarrow}{u}$. Instead, we need to consider positions $p_0, \ldots, p_{m-1}$ in the loop of $\overset{\twoheadrightarrow}{u}$ and find $i, j$
such that $\cC(i) = \cC(j)$ and \emph{$\overset{\twoheadrightarrow}{u}$ is at the same position in its loop} at $i$ and at $j$.

Observe that the loop at  lines (11--14) terminates after at most $m\cdot |\times_{X\in\cV}\cR(X)|$, where $\times_{X\in \cV}\cR(X)$ is the set of all possible assignments of the variables in $\cV$. This is also the upper bound on $y$ (the length of loop of $\cC^{\vec{Y}(\vec{n})\leftarrow \vec{y}}$).

% \maksim{minor comment: I think it should be $X\in \cV$, because $\cC$ contains only $\cV$, and we keep $\cU$ separately.}

Given a periodic computation $\cC$, a $\mathsf{PLTL}$ formula $\varphi$ and a natural number $i$, the \textit{path} model-checking problem is to decide whether $(\cC, i)\vDash_{\PLTL} \varphi$. The future modalities $\bigcirc\psi$ and $\chi\U\psi$ can be solved straightforwardly by a standard labelling algorithm for $\mathsf{LTL}$ (e.g., \cite{GorankoBook}) in time $\cO(|\cC|\cdot |\varphi|)$. However, this  approach does not extend to the full $\mathsf{PLTL}$. % The feature of any ultimately periodic computation $\cC$ is that future is deterministic but past is not, that is any state has one successor but some states have two (or zero) predecessors. 
We therefore use the technique presented in \cite{MARKEY200287}. Let $h_P(\varphi)$ denote the \textit{past}-temporal height of $\varphi$, which is the maximum number of nested \textit{past} modalities in $\varphi$.

% Given $\cC^{\vec{Y}(\vec{n})\leftarrow \vec{y}}$, for every $j\in [0, (?)]$, we want to construct the set 
% $$label^{\vec{Y}(\vec{n})\leftarrow \vec{y}}[j] := \{\psi \in \{\psi_1,\dots,\psi_k,\neg\psi_1,\dots,\neg\psi_k\} \mid \cC^{\vec{Y}(\vec{n})\leftarrow \vec{y}}(j)\vDash \psi\}$$

\begin{lemma}[Markey \shortcite{MARKEY200287}]\label{lemma:PLTLheight} For any periodic $\cC$ of type $(x, y)$, $\varphi\in \PLTL$ and $k\geq x+h_{P}(\varphi)\cdot y: (\cC, k)\vDash_{\PLTL} \varphi \text{ iff } (\cC, k+y)\vDash_{\PLTL} \varphi$. 
\end{lemma}

This lemma states that after some initial segment of $\cC$ past modalities in $\varphi$ cannot distinguish how many times the loop has been repeated. In other words, in order to verify if $(\cC, i)\vDash \varphi$ for any $i$, it is sufficient to check only the first $x+(h_{P}(\varphi)+1)\cdot y$ elements of the computation $\cC$, and if $i>x+(h_P(\varphi)+1)\cdot y$, then it is sufficient to find $k=(i-(x+(h_P(\varphi)+1)\cdot y)\mod y$ and check $(\cC, (k+x+h_P(\varphi)\cdot y))\vDash \varphi$.

\begin{corollary}[Path model-checking]\label{corollary:pathMC} 
    Given an ultimately periodic $\cC$ of type $(x, y)$, $\varphi\in \PLTL$ and $i\in \mathbb{N}$, path model-checking $(\cC, i)\vDash_{\PLTL} \varphi$ can be done in time $\cO((x+(h_{P}(\varphi)+1)\cdot y)\cdot |\varphi|)$. 
\end{corollary}

Original proofs of \Cref{lemma:PLTLheight} and \Cref{corollary:pathMC} can be found in \cite{MARKEY200287}. The only difference with our approach is that we use atomic statements $(X=x)$ instead of atomic propositions. Now we are ready to establish the main result. 

\begin{theorem} $\mathsf{CPLTL}$ model-checking is in \textsc{P}. 
\end{theorem}
\begin{proof} Let our input be $((\cM, \overset{\twoheadrightarrow}{u}, \vec{v}, t), \varphi)$ where $\overset{\twoheadrightarrow}{u}$ has the prefix length $n$ and the loop length $m$.
Note that any $\CPLTL$ formula is a boolean combination of intervention subformulas of the form $[\vec{Y}(\vec{n})\leftarrow \vec{y}]\psi$, where $\psi\in\PLTL$.
For every intervention subformula $\chi'=[\vec{Y'}(\vec{n'})\leftarrow \vec{y}]\psi'$ of $\varphi$ 
we generate a computation $\cC^{\vec{Y'}(\vec{n'})\leftarrow \vec{y'}}$ using \Cref{alg:compute} and verify if the $\PLTL$ formula $\psi'$ holds at $(\cC^{\vec{Y'}(\vec{n'})\leftarrow \vec{y'}}, t)$. By \Cref{def:semanticsCPLTL}, $(\cC^{\vec{Y'}(\vec{n'})\leftarrow \vec{y'}}, t)\vDash_{\PLTL} \psi'$ iff $(\cM, \overset{\twoheadrightarrow}{u}, \vec{v}), t\vDash \chi'$. Finally, it remains to substitute the truth values of all subformulas in $\varphi$. The overall procedure requires $\cO(|Sub(\varphi)|)$ calls to \Cref{alg:compute}, and each call takes $\cO((\max(|\overset{\twoheadrightarrow}{u}|, |\varphi|)+|\overset{\twoheadrightarrow}{u}|)\cdot|\cM|)$ steps to generate $\cC'$. Checking each $(\cC', t)\vDash_{\PLTL} \psi'$ can be done in $\cO((n+(h_{P}(\varphi)+1)\cdot m)\cdot |\varphi|)$ by \Cref{corollary:pathMC}, where $\cO(|h_{P}(\varphi)|)=\cO(|Sub(\varphi)|)$. The final substitution takes $\cO(|Sub(\varphi)|)$ more steps. Thus, $\CPLTL$ model-checking problem is solvable in polynomial time.
\end{proof}

\section{Temporal Causal Models with Arbitrary Delays}

Now we demonstrate that TSEM framework is not limited to 1-step delays in causal dependencies and can easily be extended to accommodate arbitrary long delays with only minor modifications.

As before, we assume that $\cU$ and $\cV$ are finite sets of exogenous and endogenous variables respectively, and $\cR$ is a range function associating with every variable $X_i\in \mathcal{V}$ a
non-empty \textit{finite} set $\mathcal{R}(X_i)=\{x_1, \dots, x_k\}$ of possible values. The values of variables may change over time, and these changes are governed by temporal structural equations, which describe how the current value of each $X_i\in \cV$ depends on the previous values of
%bsl X_i can depend on itself
%other 
variables in $\cV$. Since the time lags of causal dependencies can now be arbitrary, we use superscripts to denote these temporal delays;  for example, if the current value of a variable $Y$ depends on the value of $X$ 1 step ago and on the value of $Z$ 3 steps ago, we say that $X^{-1}$ and $Z^{-3}$ are the temporal parents of $Y$.

A \textit{domain function} $\cD$ maps each variable in $\cV$ to a subset of $\bigcup\limits_{0<t\leq \xi}\{X_i^{-t} \mid X_i\in \cU\cup\cV\}$, where 
%bsl introduced informally above
%We also call $\cD(Y)$ a set of temporal parents of $Y$. 
$\xi \in \mathbb{Z}^+$ is the maximal temporal delay in the model, i.e., the maximal delay with which changes in one variable may affect the value of another variable.  

A causal model $\cM$ consists of a signature $\cS= (\cU,\cV, \cR, \cD)$ and a set of temporal structural equations $\cF = \{\cF_Y\mid Y\in \cV\}$. Each structural equation $\cF_Y$ defines how the current value of $Y$ depends on the previous values of the variables in its domain $\cD(Y)$. Formally:

\begin{definition}[TSEM]\label{def:TSEM2} A %Temporal Causal Model, also known as 
Temporal Structural Equation Model over a signature $\mathcal{S}$ is a tuple $\cM = (\mathcal{S}, \cF)$, where $\cF$ associates with every variable $Y \in \mathcal{V}$ a function 
$$\mathcal{F}_Y:\prod\limits_{X^{-t}_i \in \cD(Y)}\cR(X^{-t}_i)\longrightarrow\cR(Y)$$
\end{definition}
%
%\noindent Note that each $X_i^{-t} \in \cD(Y)= \bigcup\limits_{0<t\leq \tau}\{X_1^{-t}, X_2^{-t}, \dots\}$, so $Y$ may depend on the same variable $X_i$ (including the case $X_i=Y$) at different time points. For example, $\cF_Y$ may describe an equation $Y:= (Y^{-1}+Y^{-3}+Z^{-1})\times Z^{-5}$. 
\noindent Note that $Y$ may depend on the same variable $X_i$ (including the case $X_i=Y$) at different delays. For example, $\cF_Y$ may describe an equation $Y:= (Y^{-1}+Y^{-3}+Z^{-1})\times Z^{-5}$. 

Similarly to 1-step models, the changes in exogenous variables are represented as a temporal context $\overset{\twoheadrightarrow}{u}$ given as an input, while the changes in the values of variables $\cV$ are induced by a temporal structural equations $\cF$ and represented as a computation $\cC$.  
%bsl
A computation is a series  of \emph{configurations} of $\cM$, i.e., complete assignments of values to variables $\cV$, denoted  $\vec{v}$. We also use $\vec{v}^{|X}$ to denote $\vec{v}$ restricted to $X \in \cV$. A sequence of complete assignments, which we call a \emph{history}, is denoted $\overset{\twoheadrightarrow}{v} = (\vec{v}_{0}, \dots, \vec{v}_{n})$. We use $\overset{\twoheadrightarrow}{v}[-i]$ to refer to  $i$'s last element of $\overset{\twoheadrightarrow}{v}: \overset{\twoheadrightarrow}{v}[-i] = \vec{v}_{(n-i+1)}$, e.g. $\overset{\twoheadrightarrow}{v}[-1]$ denotes the last element of $\overset{\twoheadrightarrow}{v}$.  We also use $d(Y)$ to denote the maximal temporal depth of the parents of $Y\in \cV$, i.e., $d(Y) = \max (t\mid X_i^{-t}\in \cD(Y))$. Given $|\overset{\twoheadrightarrow}{v}|\geq d(Y)$, we use $\overset{\twoheadrightarrow}{v}^{|\cD(Y)}$ to denote the function mapping each $X_i^{-t}\in \cD(Y)\cap\cV$ to $\overset{\twoheadrightarrow}{v}[-t]^{|X_i}$. Similarly, for $\overset{\twoheadrightarrow}{u}[0:n]$ with $n\geq d(Y)$, $\overset{\twoheadrightarrow}{u}[0:n]^{|\cD(Y)}$ maps each $X_i^{-t}\in \cD(Y)\cap\cU$ to $\overset{\twoheadrightarrow}{u}[-t]^{|X_i}$

Given a history $\overset{\twoheadrightarrow}{v} = [0, \dots, \vec{v}_{n}]$ and an initial segment of a temporal context $\overset{\twoheadrightarrow}{u}[0:n]$, we define a \emph{call} to $\cF_X$ with input $(\overset{\twoheadrightarrow}{u}[0:n],\overset{\twoheadrightarrow}{v})$ as 

$$\cF_X(\overset{\twoheadrightarrow}{u}[0:n],\overset{\twoheadrightarrow}{v})=\begin{cases} \overset{\twoheadrightarrow}{v}[-1]^{|X}; \text{ if } n<d(X) \\
\cF_X((\overset{\twoheadrightarrow}{u}[0:n], \overset{\twoheadrightarrow}{v})^{|\cD(X)}); \text{ otherwise } 
    
\end{cases} $$

Informally, if the history $\overset{\twoheadrightarrow}{v}$ and the context $\overset{\twoheadrightarrow}{u}$ are long enough, then the value of $X$ is determined according to its structural equation $\cF_X$. However, if $\overset{\twoheadrightarrow}{v}$ and $\overset{\twoheadrightarrow}{u}$ are shorter than $d(X)$, so not all temporal parents of $X$ are present in $(\overset{\twoheadrightarrow}{u}[0:n], \overset{\twoheadrightarrow}{v})$, then $X$ keeps the same value as in the last moment $\overset{\twoheadrightarrow}{v}[-1]$ of the history. 

\begin{definition}[Computation]\label{def:computation2} Given a causal model $\cM$, a temporal context $\overset{\twoheadrightarrow}{u}$ and an initial configuration $\vec{v_0}$, a computation $\cC$ is a function 
$$\cC: \mathbb{N}\longrightarrow \prod\limits_{X\in \cV} \cR(X)$$
mapping $\mathbb{N}$ to configurations of $\cV$, such that $\cC(0)=\vec{v_0}$ and for all $i>0$: 
$$\cC(i) = \prod\limits_{X\in \cV} \cF_X\bigl(\overset{\twoheadrightarrow}{u}[0:i-1], \cC[0:i-1]\bigr)$$

Here $\cC(i)$ denotes the complete assignment of $\cV$ at time step $i$, and $\overset{\twoheadrightarrow}{u}[0:i-1]$ and $\cC[0:i-1]$ denote the initial segments of $\overset{\twoheadrightarrow}{u}$ and $\cC$ respectively, i.e., a finite sequences of complete assignments of $\cU$ and $\cV$. 
%where by $\cC^{|\cD(X)}(i-1)$ we mean the configuration $\cC(i-1)$ restricted to the variables in the domain of $X$. 
\end{definition}
%
%\footnote{Even though a configuration $\vec{v}$ is a complete assignment of infinitely many variables $\cV$, in the rest of the paper we assume that $\vec{v}$ is finitely presentable. In other words, we assume that every variable $Y\in \cV$ has a special \emph{default} value $\#\in \cR(Y)$, and all variables whose values are not mentioned in $\vec{v}$ has their default values.} 
 So, as before, a computation $\cC$ starts from a given initial configuration $\vec{v_0}$ and repeatedly `updates' the values of $\cV$ according to $\cF$. Finally, we interpret the time-indexed interventions similarly to 1-step models:

\begin{definition}[Updated Computation]\label{def:updatedcomputation2} Given an intervention $int=\vec{Y}(\vec{n})\leftarrow \vec{y}$ and $(\cM, \overset{\twoheadrightarrow}{u}, \vec{v_0})$, an \emph{updated} computation $\cC^{int}$ is defined as $\cC^{int}(0)=\vec{v_0}^{int}$, and $\forall i>0\forall X\in \cV$: 
$$\cC^{int}(i)|_X=\begin{cases}
			x', \text{ if } X(i)\gets x' \in \vec{Y}(\vec{n})\leftarrow \vec{y}\\
            \cF_X(\overset{\twoheadrightarrow}{u}[0: i-1], \cC^{int}[0:i-1]), \text{ otherwise}
		 \end{cases}$$
\end{definition}

Thus, the only modifications that we introduced with respect to 1-step lagged models are (1) the domain function $\cD$ for notational convenience; (2) temporal structural equations take into account a finite history of changing values rather than the values from the previous time instance; (3) we assumed that in the initial segment of the computation, when the history is not long enough to determine the current value of some variable, this variable keeps the same value as on the previous step. 

Although these models seem more expressive, they are in fact equivalent to 1-step models. In other words, any model with arbitrary delays can be simulated with a 1-step lagged model.

\begin{theorem} For any arbitrary lagged TSEM $\cM^d$ with a maximal temporal delay $d$, there exists an equivalent 1-step lagged TSEM $\cM^1$.
\end{theorem}
\begin{proof}

First, for a given $\cM^d = ((\cU, \cV, \cR, \cD), \cF)$ we construct a corresponding $\cM^1 = ((\cU^1, \cV^1, \cR^1), \cF^1)$. 

Let $Indexed(\cU) = \{U_i^{t} \mid  \exists Y\in \cV, U_i^t\in \cD(Y)\}$ and $Indexed(\cV) = \{X_i^{t} \mid  \exists Y\in \cV, X_i^t\in \cD(Y)\}$ be sets of time-indexed exogenous and endogenous variables respectively that appear in the domain function $\cD$ of at least one variable from $\cV$. These variables will simulate arbitrary time lags in $\cM^1$. The signature $\cS^1=(\cU^1, \cV^1, \cR^1)$ is then defined as
\begin{itemize}
    \item $\cU^1 = \cU$;
    \item $\cV^1 = \cV \cup Indexed(\cV)  \cup Indexed(\cU)$; 
    \item for any non-indexed variable $X_i \in \cU \cup \cV$, $\cR^1$ is identical to $\cR$: $\cR^1(X_i) = \cR(X_i)$;
    \item for an any time-indexed variable $X_i^t \in Indexed(\cU) \cup Indexed(\cV)$, $\cR^1(X_i^t) = \cR(X_i)\cup \{\#\}$. So, the range $\cR^1(X_i^t)$ contains all the values of the corresponding $X_i$ in $\cM^d$ with an additional `undefined' value $\#$.
\end{itemize}

The idea behind this construction is to mimic arbitrary long time delays in $\cM^d$ by expanding explicit chains of time-indexed variables with 1-step delays only in $\cM^1$. So, for the indexed variables the structural equations are defined as follows. For all $X_i^t\in Indexed(\cV)\cup Indexed(\cU)$ and all $\vec{z}\in \cR(\cU\cup\cV)$:

$$
\cF^1_{X_i^t}(\vec{z}) =  \begin{cases} x, \text{where } (X_i^{t+1}=x)\in \vec{z}, \text{ for } t<-1;\\ 
x', \text{where } (X_i=x')\in \vec{z}, \text{ for } t=-1;
\end{cases}
$$

For the non-indexed variables $X_i\in \cV$ and all $\vec{z}\in \cR(\cU\cup\cV)$ we define $\cF_{X_i}^1(\vec{z})$ as follows. If all temporal parents $\cD(X_i)$ of $X_i$ are defined (i.e. assigned with non $\#$) in $\vec{z}$, then $\cF_{X_i}^1(\vec{z}) = \cF_{X_i}(\vec{z}^{|\cD(X_i)})$. Otherwise, $\cF_{X_i}^1(\vec{z})=x$, where  $(X_i=x)\in \vec{z}$. So, if the assignment of the variables from the previous step $\vec{z}$ contains enough information to determine the value of $X_i$, we assign this value, otherwise $X_i$ keeps the same value as in $\vec{z}$.

Finally, given an initial configuration $\vec{v_0}$ of $\cM^d$, the initial configuration $\vec{v_0}^1$ of $\cM^1$ agrees with $\vec{v_0}$ on the values of all non-indexed variables $X_i\in \cU\cup \cV$ and for all indexed $X_i^t\in Indexed(\cU) \cup Indexed(\cV), (Y=\#)\in \vec{v_0}^1$. 

Recall that the equivalence of two models is defined with respect to a set of shared (i.e. observed) variables $\cO \subseteq \cV_1\cap\cV_2$ contained in both models (\Cref{def:equiv:model}). In our case the set of shared variables is $\cV$, i.e. the set of all endogenous variables in $\cM^d$. 
We need to prove that for any intervention $int = \vec{Y}(\vec{n})\leftarrow \vec{y}$, where $\vec{Y}\subseteq \cV$ and for any $(\overset{\twoheadrightarrow}{u}, \vec{v_0})$, ($\cM^d, \overset{\twoheadrightarrow}{u}, \vec{v_0}$) and ($\cM^1, \overset{\twoheadrightarrow}{u}, \vec{v_0}^1$) generate computations $\cC^{int}_d$ and $\cC^{int}_1$ that are equivalent, i.e., agree on the values of all variables in $\cV$ at any step.  We use the notation $\cC_1(i) =_{\cV} C_d(i)$ to express the fact that $\forall X\in \cV, (X=x)\in \cC_1(i)$ iff $(X=x)\in \cC_d(i)$.

For step 0, $\vec{v_0}^1$ agrees with $\vec{v_0}$ by construction, so $\cC^{int}_d(0) =_{\cV} \cC^{int}_1(0)$. It remains to show that if both computations agree up to some time step $n$, i.e. $\forall 0\leq i \leq n, \cC^{int}_1(i) =_{\cV} \cC^{int}_d(i)$, then $\cC^{int}_1(n+1) =_{\cV} \cC^{int}_d(n+1)$. So, assume that $\cC^{int}_1(i) =_{\cV} \cC^{int}_d(i)$ for all $0\leq i \leq n$, and consider all variables $X\in \cV$ that do not occur in the intervention $int$. We want to prove that  $\cF_X(\overset{\twoheadrightarrow}{u}[0: n], \cC^{int}_d[0:n]) = \cF_X^1(\overset{\twoheadrightarrow}{u}(n), \cC^{int}_1(n))$. The construction of $\cF^1$ guarantees that, if all temporal parents $Y^{t}$ of $X$ are defined in $(\overset{\twoheadrightarrow}{u}[0: n], \cC^{int}_d[0:n])$, then they are defined in $(\overset{\twoheadrightarrow}{u}(n), \cC^{int}_1(n))$ as well, and $\cF_X$ and $\cF_X^1$ return the same value. At the same time, if $n< d(X)$, then both $\cF_X(\overset{\twoheadrightarrow}{u}[0: n], \cC^{int}_d[0:n])$ and $\cF_X^1(\overset{\twoheadrightarrow}{u}(n), \cC^{int}_1(n))$ return the same value $x'$ of $X$ as on the previous time step $(X=x')\in \cC^{int}_1(n)$. So, for any $X\in \cV$ that do not occur in the intervention $int$, it holds that $(X=x)\in \cC_d^{int}(n+1)$ iff $(X=x)\in \cC_1^{int}(n+1)$. Finally, for any $Y(n+1)\gets y' \in int$ it holds that $(Y=y')\in \cC_d^{int}(n+1)$ and $(Y=y')\in \cC_1^{int}(n+1)$ by Definitions $\ref{def:updatedcomputation}$ and $\ref{def:updatedcomputation2}$. Thus,  $\cC_d^{int}(n+1) =_{\cV} \cC_1^{int}(n+1)$, which completes the proof.
\end{proof}

%\suggestion{Note that $|\cU^1| = \cO(|\cU|)$, $|\cV^1| = \cO(d\cdot(|\cV|+|\cU|))$ and $|\cF^1|=\cO(d\cdot|\cF|)$, where $d$ is the temporal depth of $\cM$. Thus,
Note that  the size of the resulting model $\cM^1$ constructed above is polynomial in the size of the original $\cM$:  $ |\cM^1| = \cO(|\cM|^2)$.

\section{Related Work}

The computational complexity of verifying actual causation in SEMs have been studied by many researchers, for example, \citet{Halpern2015}, \citet{Gladyshev_ECAI2023} and \citet{Lorini2024}.

Although \citet{Beckers2018} argued that accounting for temporal information is crucial for actual causation judgments, to the best of our knowledge the framework proposed in this paper is the first attempt to integrate LTL-style temporal reasoning into actual causality settings. 

Previous work on combining SEM models with temporal reasoning has focused on causal discovery, and several methods for analysing time-series data with SEMs have been developed, e.g., \cite{Assaad_2022,hyvärinen2023}. 
%At the same time, in the field of Causal Discovery various methods for analyzing and predicting time-series have been developed recently \cite{Assaad_2022,hyvärinen2023}. 
These methods typically use Full Time Causal Graphs to represent and reason about time-series using SEMs. Full Time Causal Graphs use time-indexed variables \cite[Ch. 10]{causal_timeseries} and thus potentially require specifying infinitely many structural equations. However, if all the causal relations remain constant over time, a Full Time Causal Graph can be abbreviated with a (finite) Window Causal Graph or a (finite) Summary Causal Graph \cite{Assaad_2022}. Though the dependency graphs of our models resemble both Window and Summary graphs in the way they interpret cyclic dependencies, there are important differences. Firstly,
conventional models of time-series allow contemporaneous dependencies, i.e., a (time-indexed) variable $X_t$ may depend on the value of another variable $Y_t$ at the same time step $t$. This modelling choice is usually driven by the observational limitations in the field of causal discovery, e.g., the sampling frequency of the time series may not be able to separate causes and effects. In the field of actual causality, in contrast, we assume that a given causal model provides a complete representation of reality, and thus we assume that the unit interval is sufficiently small to separate causes and subsequent effects. So, contemporaneous relations do not occur in our framework. 
Secondly, both Full Time and Window graphs allow arbitrary `lagged' dependencies, i.e., a variable $X$ at time $t$ may (directly) depend on another variable $Y$ at time $t-n$ for arbitrary $n$. We primarily focused on 1-step lagged models due to notational convenience, however this is not a fundamental limitation of the proposed framework. As we demonstrated, (1) arbitrary lagged models can be defined straightforwardly in our setting, and (2) for any arbitrary lagged TSEM there exists an equivalent polynomial sized 1-step lagged TSEM.

% *** Overall, our models are not abstractions of some underlying models, they don't use time-indexed variables and thus keep sets of variables, their ranges and equations finite. ***

In contrast to causal discovery, where cyclic causal models are sometimes used, e.g. \cite{Bongers_etal_2021}, current research in actual causality is mostly focused on acyclic (recursive) models. To the best of our knowledge, there has been relatively little work which considers non-recursive models, e.g., \citet{halpern2000axiomatizing} studies axiomatizations for different classes of SEM's, in particularly non-recursive ones, and \citet[Ch. 2.7]{HalpernBook} discusses the definition of an \textit{actual cause} in such models. Finally, \citet{Halpern2022NonRecurive} introduce so-called \textit{Generalized} SEMs (GSEM) %(suitable for reasoning about models with infinitely many variables with possibly infinite range) 
and axiomatize  different classes of GSEMs, including non-recursive ones. Note that in GSEMs structural equations $\cF$ are replaced with $\mathbf{F}$ mapping contexts and interventions to sets of outcomes. We argue that, by staying closer to original formalism, our temporal framework accommodates non-recursive models without significant adjustments.

Defining the equivalence of causal models is another well-recognized problem both in causal discovery and actual causality. 
%Natasha: times are in the wrong order. Is the reference to Pearl wrong? I replaced Verma_Pearl_2022 with Pearl2000.
%We have already mentioned the work of  \citet{Beckers2021eq}, but the problem was recognized already by Pearl \cite{Verma_Pearl_2022}. A similar notion of causal consistency has also been studied recently \cite{rubenstein2017causal}.
%
We have already mentioned the work of  \citet{Beckers2021eq}, but the problem was recognized already by \citet{VermaPearl1990}. A similar notion of causal consistency has also been studied in \cite{rubenstein2017causal}. 

Similarly to equivalence, causal consistency, intuitively guarantees that two models agree in their predictions of the effects of interventions. \citet{rubenstein2017causal} introduced the notion of exact transformations between SEMs preserving causal consistency,  and \citet{willig2023do} proposed another type of transformations, called consolidating mechanisms, to transform large-scale SEMs into smaller, computationally efficient ones. Both of these  approaches to consistency-preserving transformations are applicable to time-series models. We believe studying our models as transformations of existing time-series models based on \cite{rubenstein2017causal,willig2023do} is an interesting direction for future work. 

% Natasha: moved to earlier
%Finally, computational complexity of verifying actual causation in SEMs have been studied by \citet{Halpern2015}, \citet{Gladyshev_ECAI2023} and \citet{Lorini2024}.

\section{Discussion}

We have proposed a novel conceptual interpretation of existing formalisms in the field of actual causality. While causal models have been developed as a useful tool for representing static dependencies between the variables, we demonstrate that this formalism (with minor modifications) can be treated as a mechanism able to transform a time series of exogenous values into a time series of endogenous ones. Though our approach does not allow us to extract temporal information from existing static models which were not designed with this purpose, it provides new insights and techniques for (temporal) structural equation modelling. 

We make a number of technical contributions. We introduced the core concept of a computation, which treats structural equations as `time-lagged' and allows us to `unwind' a causal scenario $(\cM, \overset{\twoheadrightarrow}{u}, \vec{v})$ into a time series $\cC$ of the assignments to $\cV$. In addition, we proposed the logic $\CPLTL$, which combines the temporal logic $\LTL$ with causal time-sensitive interventions. Finally, we introduced new notions of temporal equivalence for causal models and showed that the model-checking problem for $\CPLTL$ is in \textsc{P}.

We believe there are a number of interesting directions for future work. First, in this paper we discussed only finite models, however the extension of our framework to models with infinitely many variables (and potentially infinite range) seems to be straightforward. Secondly, the problems of defining a (bi)simulation relation between temporal causal models (which would imply their temporal equivalence) as well as defining weaker notions of temporal equivalence are left open. Finally, another promising direction for future work is adaptation of our framework to probabilistic settings, when we have a probability distribution on temporal contexts or non-deterministic structural equations as recently proposed by \citet{beckers2024nondeterministiccausalmodels}.

\bibliography{aaai25}
%\input{appendix}

% \newpage
% \appendix
%\input{appendix}

\end{document}